\documentclass[sigconf]{acmart}

\usepackage{graphicx}
\usepackage{amsmath}
\usepackage{algorithm}
\usepackage{algpseudocode}
\usepackage{multirow}
\usepackage{booktabs}
\usepackage{wrapfig}
\usepackage{caption}
\usepackage{subcaption}

\algrenewcommand\algorithmicrequire{\textbf{Input:}}
\algrenewcommand\algorithmicensure{\textbf{Output:}}

\usepackage{float}

\makeatletter
\newcommand*{\indep}{%
	\mathbin{%
		\mathpalette{\@indep}{}%
	}%
}
\newcommand*{\nindep}{%
	\mathbin{% The final symbol is a binary math operator
		\mathpalette{\@indep}{\not}% \mathpalette helps for the adaptation
	}%
}
\newcommand*{\@indep}[2]{%
	\sbox0{$#1\perp\m@th$}%box 0 contains \perp symbol
	\sbox2{$#1=$}%box 2 for the height of =
	\sbox4{$#1\vcenter{}$}% box 4 for the height of the math axis
	\rlap{\copy0}% first \perp
	\dimen@=\dimexpr\ht2-\ht4-.2pt\relax
	\kern\dimen@
	{#2}%
	\kern\dimen@
	\copy0 %
}

\newtheorem{definition}{Definition}

\AtBeginDocument{%
  \providecommand\BibTeX{{%
    \normalfont B\kern-0.5em{\scshape i\kern-0.25em b}\kern-0.8em\TeX}}}

\begin{document}

\title{Disentangled Representation with Causal Constraints for Counterfactual Fairness}


\author{Ziqi Xu}
\affiliation{%
  \institution{University of South Australia}
  \city{Adelaide}
  \country{Australia}}
\email{ziqi.xu@mymail.unisa.edu.au}

\author{Jixue Liu}
\affiliation{%
	\institution{University of South Australia}
	\city{Adelaide}
	\country{Australia}}
\email{jixue.liu@unisa.edu.au}

\author{Debo Cheng}
\affiliation{%
	\institution{University of South Australia}
	\city{Adelaide}
	\country{Australia}}
\email{dobo.cheng@unisa.edu.au}

\author{Jiuyong Li}
\affiliation{%
	\institution{University of South Australia}
	\city{Adelaide}
	\country{Australia}}
\email{jiuyong.li@unisa.edu.au}

\author{Lin Liu}
\affiliation{%
	\institution{University of South Australia}
	\city{Adelaide}
	\country{Australia}}
\email{lin.liu@unisa.edu.au}

\author{Ke Wang}
\affiliation{%
	\institution{Simon Fraser University}
	\city{Burnaby}
	\country{Canada}}
\email{wangk@cs.sfu.ca}

\renewcommand{\shortauthors}{Z. Xu et al.}


\begin{abstract}
Much research has been devoted to the problem of learning fair representations; however, they do not explicitly the relationship between latent representations. In many real-world applications, there may be causal relationships between latent representations. Furthermore, most fair representation learning methods focus on group-level fairness and are based on correlations, ignoring the causal relationships underlying the data. In this work, we theoretically demonstrate that using the structured representations enable downstream predictive models to achieve counterfactual fairness, and then we propose the Counterfactual Fairness Variational AutoEncoder (CF-VAE) to obtain structured representations with respect to domain knowledge. The experimental results show that the proposed method achieves better fairness and accuracy performance than the benchmark fairness methods.
\end{abstract}

\maketitle

\section{Introduction}
Machine learning algorithms have gradually penetrated into our life~\cite{mehrabi2021survey} and have been applied to decision-making for credit scoring~\cite{kruppa2013consumer}, crime prediction~\cite{kim2018crime} and loan assessment~\cite{cocser2019predictive}. The fairness of these decisions and their impact on individuals or society have become an increasing concern. Some extreme unfair incidents have appeared in recent years. For example, COMPAS, a decision support model that estimates the risk of a defendant becoming a recidivist was found to predict higher risk for black people and lower risk for white people~\cite{brennan2009evaluating}; Google Photos are classifying black people as primates~\cite{zhang2015google}; Facebook users receive a recommendation prompt when watching a video featuring blacks, asking them if they’d like to continue to watch videos about primates~\cite{mac2021facebook}. These incidents indicate that the machine learning models may become a source of unfairness, which may lead to serious social problems. Since most models are trained with data, which will lead to unfair decisions due to discrimination in the training data. Therefore, the key issue for solving unfair decisions becomes whether we can eliminate these biases embedded in the data through algorithms~\cite{lee2019algorithmic}.

To obtain unbiased decisions, many methods~\cite{zemel2013learning,LouizosSLWZ15,madras2018learning,madras2019fairness,song2019learning,creager2019flexibly,wang2019balanced,park2021learning,gitiaux2021learning} are proposed to learn fair representations through two competing goals: encoding data as much as possible, while eliminating any information that transfers through the sensitive attributes. To separate the information from sensitive attributes, various extensions of Variational Autoencoder (VAE) consider minimising the mutual information among latent representations~\cite{LouizosSLWZ15,creager2019flexibly,song2019learning,park2021learning}. For example, \citet{creager2019flexibly} introduced disentanglement loss into the VAE objective function to decompose observed attributes into sensitive latents and non-sensitive latents to achieve subgroup level fairness; \citet{park2021learning} improved the above methods and proposed the mutual attribute latent (MAL) to retain only beneficial information for fair predictions.
\begin{figure*}[t]
	\begin{subfigure}[b]{0.21166\textwidth}
		\centering
		\includegraphics[scale=0.5]{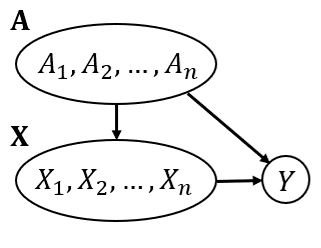}
		\caption{}
		\label{pic:intro1}
	\end{subfigure}
	\begin{subfigure}[b]{0.36755\textwidth}
		\centering
		\includegraphics[scale=0.5]{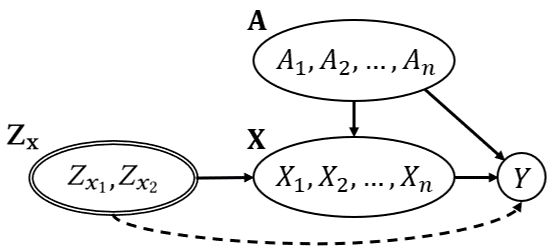}
		\caption{}
		\label{pic:intro2}
	\end{subfigure}
	\begin{subfigure}[b]{0.37069\textwidth}
		\centering
		\includegraphics[scale=0.5]{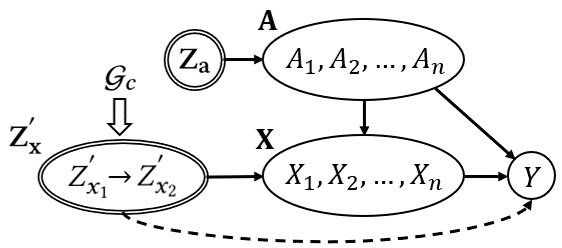}
		\caption{}
		\label{pic:intro3}
	\end{subfigure}
	\caption{(a) Causal graph for an example. (b) The process of existing work on learning fair representations to make predictions. (c) The process of our work. $\mathbf{A}$ is the set of sensitive attributes; $\mathbf{X}$ is the set of other observed attributes; $\mathbf{Z_a}$ is the representation of $\mathbf{A}$; $Y$ is the target attribute. The dotted line represent the prediction process that use the learnt representations obtained by the methods. $\mathbf{Z_x}$ is the representation of $\mathbf{X}$; $\mathbf{Z^{'}_{x}}$ is the structured representation of $\mathbf{X}$ with respect to the conceptual level causal graph $\mathcal{G}_c$. Both $\mathbf{Z_x}$ and $\mathbf{Z^{'}_{x}}$ do not contain sensitive information that transform through the path $\mathbf{A} \to \mathbf{X}$ since these are learned based on the constraints mentioned in the different methods.}
	\label{pic:intro}
\end{figure*}

All existing works of learning fair representations make the assumption that all observed attributes in the real-world can be represented by a number of latent representations. Nevertheless, the latent representations may have causal relationships among them. Let us consider an example where we aim to predict a person's salary using some observed attributes. Following the domain knowledge, we know that a person's salary is determined by two semantic concepts, intelligence and career respectively. We also note that a person's intelligence determines their career with high probability, which can be expressed as a conceptual level causal graph $\mathcal{G}_c$, that is, $Intelligence \to Career$. Figure~\ref{pic:intro1} shows the causal graph that is learnt from the collected data, while the data itself is biased since the set of sensitive attributes $\mathbf{A}$ can affect the target attribute $Y$. 

The existing methods~\cite{LouizosSLWZ15,creager2019flexibly} follow Figure~\ref{pic:intro2} to achieve fair predictions. Specifically, this method uses $\mathbf{Z_{x}}$ to represent the ``concept''  as mentioned while ensuring $\mathbf{Z_{x}}$ do not contain sensitive information that transfer through the path $\mathbf{A} \to \mathbf{X}$. However, this method may not satisfy the domain knowledge since there are causal relationships within these ``concepts''. Therefore, we need a method as shown in Figure~\ref{pic:intro3} that not only ensures the representation of observed attributes with no sensitive information but also retains causal relationships with respect to domain knowledge. We note that our method builds on the premise that $\mathcal{G}_c$ is available, and we believe this assumption is valid. Fairness issues require humans to guide algorithms, and the causal graph should be given by humans rather than given by machine learning~\cite{carey2022causal}. Compared with the complete version of the causal graph (i.e., a causal graph containing causal relationships between all observed attributes), $\mathcal{G}_c$ only covers the relationship between these ``concepts'' and is easier to obtain expert consensus.

On the measurement of fairness, all fair representation learning methods use fairness metrics based on correlation, including the VAE-based methods~\cite{LouizosSLWZ15,creager2019flexibly,song2019learning,park2021learning}. It is well known that correlation or more generally association does not imply causation. Recent studies~\cite{pearl2009causal,pearl2009causality} have shown that quantifying fairness based on correlation may produce higher deviations. Counterfactual fairness is a fundamental framework based on causation. With counterfactual fairness,  a decision is fair towards an individual if it is the same in the actual world and in the counterfactual world when the individual belonged to a different demographic group. 

In this paper, we follow the counterfactual fairness and propose a VAE-based unsupervised fair representation learning method, namely Counterfactual Fairness Variational AutoEncoder (CF-VAE). We take all the observed attributes (except target attribute $Y$) as input, and disentangle the latent representations into $\mathbf{Z_a}$ and $\mathbf{Z_{x}^{'}}$. With the causal constraints, $\mathbf{Z_{x}^{'}}$ retains the causal relationships with respect to domain knowledge while containing no sensitive information. We prove that $\mathbf{Z_{x}^{'}}$ is suitable to train the counterfactually fair predictive models. To the best of our knowledge, this work is the first unsupervised method that uses VAE-based techniques to learn the fair representations that enable counterfactual fairness for downstream predictive models. We make the following contributions in this paper: 
\begin{itemize}
	\item We propose CF-VAE, a novel VAE-based unsupervised counterfactual fairness method. CF-VAE can learn structured representations with no sensitive information and retain causal relationships with respect to the conceptual level causal graph determined by domain knowledge.

	\item We theoretically demonstrate that the structured representations obtained by CF-VAE are suitable for training counterfactually fair predictive models. 
	
	\item We evaluate the effectiveness of the CF-VAE method on real-world datasets. The experiments show that CF-VAE outperforms existing benchmark fairness methods in both accuracy and fairness.
\end{itemize}

The rest of this paper is organised as follows. In Section 2, we discuss background knowledge, including our notations. The details of CF-VAE are shown in Section 3. In Section 4, we discuss the experiment results. In Section 5, we discuss related works. Finally, we conclude this paper in Section 6.

\section{Background}
We use upper case letters to represent attributes and boldfaced upper case letters to denote the set of attributes. We use boldfaced lower case letters to represent the values of the set of attributes. The values of attributes are represented using lower case letters. 

Let $\mathbf{A}$ be the set of sensitive attributes, which should not be used for predictive models; $\mathbf{X}$ be the set of other observed attributes, which may have causal relationships with $\mathbf{A}$; $\mathbf{V}$ be the set of all observed attributes, i.e., $\mathbf{V}=\{\mathbf{A},\mathbf{X}\}$; $Y$ be the target attribute that may have causal relationships with attributes in $\mathbf{A}$ and $\mathbf{X}$. We use $\widehat{Y}(\cdot)$ to represent the predictor. 

$\mathcal{G}_c$ is the conceptual level causal graph and represents domain knowledge. The nodes shown in $\mathcal{G}_c$ are ``concepts'', each of which represents a set of observed attributes that have similar meanings.  Each ``concept'' has causal relationships with the other  ``concepts''. For example, $Intelligence$ is a  ``concept'' in $\mathcal{G}_c$ and it may represent several observed attributes that have similar meanings, including $GPA$, $Education~level$ and $Major$. 

We define that $\mathbf{Z_a}$ is the representation of $\mathbf{A}$; $\mathbf{Z_x}$ is the representation of $\mathbf{X}$ without embedding causal relationships; $ \mathbf{Z_x^{'}}$ is a structured version of $\mathbf{Z_x}$ under the causal constraints of domain knowledge and does not contain sensitive information.

\subsection{Counterfactual Fairness}\label{Counterfactual}
In this paper, a causal graph is used to represent a causal mechanism. In a causal graph, a directed edge, such as $V_j \rightarrow V_i$ denotes that $V_j$ is a parent (i.e., direct cause) and we use $pa_i$ to denote the set of parents of $V_i$. We follow Pearl's~\cite{pearl2009causality} notation and define a causal model as a triple $(\mathbf{U},\mathbf{V},\mathbf{F})$: $\mathbf{U}$ is a set of the latent background attributes, which are the factors not caused by any attributes in the set $\mathbf{V}=\{\mathbf{A},\mathbf{X}\}$; $\mathbf{F}$ is a set of deterministic functions, $V_i = f_i(pa_i,U_{pa_i})$, such that $pa_i\subseteq \mathbf{V}\backslash\{V_i\}$ and $U_{pa_i}\subseteq \mathbf{U}$. Such equations are also known as structural equations~\cite{Bollen89}. Besides, some commonly used definitions in graphical causal modelling, such as faithfulness, $d$-separation and causal path can be found in~\cite{spirtes2000causation, richardson2002ancestral, pearl2009causality}.

With the causal model $(\mathbf{U},\mathbf{V},\mathbf{F})$, we have the following definition of counterfactual fairness:
\begin{definition}[Counterfactual Fairness~\cite{kusner2017counterfactual}]
	\label{def:Counterfactual}
	Predictor $\widehat{Y}(\cdot)$ is counterfactually fair if under any context $\mathbf{X}=\mathbf{x}$ and $\mathbf{A}=\mathbf{a}$,
	\begin{equation}
		\begin{aligned}
			P(\widehat{Y}_{\mathbf{A}\leftarrow \mathbf{a}}(\mathbf{U}) &=y~|~\mathbf{X}=\mathbf{x}, \mathbf{A}=\mathbf{a}) =\\ &P(\widehat{Y}_{\mathbf{A}\leftarrow \bar{\mathbf{a}}}(\mathbf{U}) =y~|~\mathbf{X}=\mathbf{x}, \mathbf{A}=\mathbf{a}),
		\end{aligned}
	\end{equation}for all $y$ and any value $\bar{\mathbf{a}}$ attainable by $\mathbf{A}$. 
\end{definition}

Counterfactual fairness is considered to be related to individual fairness~\cite{kusner2017counterfactual}. Individual fairness means that similar individuals should receive similar predicted outcomes. The concept of individual fairness when measuring the similarity of the individual is unknowable, which is similar to the unknowable distance between the real-world and the counterfactual world in counterfactual fairness~\cite{lewis2013counterfactuals}. 

\subsection{Variational Autoencoder}\label{VAE}
Variational Autoencoder (VAE) was proposed by \citet{KingmaW13}, which was originally applied to image dimensionality reduction. The objective of VAE is to maximise the Evidence Lower Bound (ELBO) $\mathcal{M}$, and derived as follows:
\begin{equation}
	\begin{aligned}
		\log p(\mathbf{V}) &\ge \mathbb{E}_{q(\mathbf{Z}|\mathbf{V})}[\log p(\mathbf{V}|\mathbf{Z})+\log p(\mathbf{Z})-\log q(\mathbf{V}|\mathbf{Z})]	\\& =: \mathcal{M}_{\text{VAE}},
	\end{aligned}
\end{equation}which can also be rewritten as follows:
\begin{equation}
	\label{eq:VAE}
	\mathcal{M}_{\text{VAE}} = \mathbb{E}_{q(\mathbf{Z}|\mathbf{V})}[\log p(\mathbf{V}|\mathbf{Z})] - D_{KL}[q(\mathbf{Z}|\mathbf{V})||p(\mathbf{Z})],
\end{equation}where $\mathbf{V}$ denotes the set of all observed attributes and $\mathbf{Z}$ denotes the set of learnt representations. The encoder $q$ encodes $\mathbf{V}$ into $\mathbf{Z}$, and the decoder $p$ reconstructs $\mathbf{V}$ from $\mathbf{Z}$. 

The first part in Equation~\ref{eq:VAE} can be considered as reconstruction error, i.e., the loss between the reconstructed and the original $\mathbf{V}$. The second part is the distribution distance between the Gaussian prior $p(\mathbf{Z}) = \mathcal{N}(0,\mathbf{I})$ and the $\mathbf{Z}$ encoded with $\mathbf{V}$. The training process of a VAE is to learn the parameters in $q$ and $p$ through the neural networks.

\citet{HigginsMPBGBML17} modified the above VAE objective function by adding a hyperparameter $\beta$ that balances latent channel capacity and independence constraints with reconstruction accuracy. Then, they devised a protocol to quantitatively compare the degree of disentanglement learnt by different models and argued that each dimension of a correctly disentangled representation should capture no more than one semantically meaningful concept. The ELBO of $\beta$-VAE is defined as:
\begin{equation}
	\begin{aligned}
		\mathcal{M}_{\beta\text{-VAE}} = \mathbb{E}_{q(\mathbf{Z}|\mathbf{V})}[\log p(\mathbf{V}|\mathbf{Z})] - \beta D_{KL}[q(\mathbf{Z}|\mathbf{V})||p(\mathbf{Z})].
	\end{aligned}	
\end{equation}  

\citet{kim2018disentangling} showed that $\mathbf{Z}$ can be considered as disentangled if each attribute in $\mathbf{Z}$, denoted as $Z_i$ is independent of each other. They minimised the total collection~\cite{watanabe1960information} of the latent representations as follows, and guaranteed disentanglement.
\begin{equation}
	\text{Total~Correction} = D_{KL}[q(\mathbf{Z})||\prod_{i=1}^{D_{\mathbf{Z}}} q(Z_i)], 
	\label{eqa:TC}
\end{equation}where $D_{\mathbf{Z}}$ is dimension of $\mathbf{Z}$. They proposed Factor-VAE by using total correction and the ELBO of Factor-VAE is defined as:
\begin{equation}
	\mathcal{M}_{\text{Factor-VAE}} = \mathcal{M}_{\text{VAE}} - \gamma D_{KL}[q(\mathbf{Z})||\prod_{i=1}^{D_{\mathbf{Z}}} q(Z_i)].
\end{equation}

\section{Proposed Method}
In this section, we first theoretically demonstrate  that learning counterfactually fair representations are feasible. Then, we propose the Counterfactual Fairness Variational AutoEncoder (CF-VAE) to obtain the structured representations for predictors to achieve counterfactual fairness.

\subsection{The Theory of Learning Counterfactually Fair Representations}
We discuss what types of representations enable downstream predictive models to achieve counterfactual fairness. Following the work in~\cite{glymour2016causal}, we define the three steps for counterfactual inference.
\begin{definition}[Counterfactual Inference~\cite{glymour2016causal}]
	\label{def:inference}
	Given a causal model $(\mathbf{U},\mathbf{V},\mathbf{F})$ and evidence $\mathbf{W}$, where $\mathbf{W} \subset \mathbf{V}$,  the counterfactual inference is the computation of probabilities $P(Y_{\mathbf{A} \leftarrow \mathbf{a}}(\mathbf{U}|\mathbf{W}=\mathbf{w}))$.
	\begin{itemize}
		\item \textbf{Abduction}: for a given prior on $\mathbf{U}$, compute the posterior distribution of $\mathbf{U}$ given the evidence $\mathbf{W}=\mathbf{w}$;
		
		\item \textbf{Action}: substitute the equations for $\mathbf{A}$ with the interventional values $\mathbf{a}$, resulting in the modified set of equations $\mathbf{F_a}$;
		
		\item \textbf{Prediction}: compute the implied distribution on the remaining elements of $\mathbf{V}$ using $\mathbf{F_a}$ and the posterior $P(\mathbf{U}|\mathbf{W}=\mathbf{w})$.
	\end{itemize}
\end{definition}

Following the work in~\cite{kusner2017counterfactual}, the implication of counterfactual fairness is described as follows:
\begin{proposition}[Implication of Counterfactual Fairness\\~\cite{kusner2017counterfactual}]
	\label{def:implication}
	Let $\mathcal{G}$ be the causal graph of the given model $(\mathbf{U},\mathbf{V},\mathbf{F})$. If there exists $\mathbf{W}$ be any non-descendant of $\mathbf{A}$, then downstream predictor $\widehat{Y}(\mathbf{W})$ will be counterfactually fair.
\end{proposition}

We extend Proposition~\ref{def:implication} to the fair representation learning and present the following theorem. We follow the similar proof process in work~\cite{kusner2017counterfactual} to prove this theorem.
\begin{theorem}
	\label{theo:002}
	Given the causal graph $\mathcal{G}$, $\mathbf{Z_a}$ is the representation of sensitive attributes $\mathbf{A}$, $\mathbf{Z^{'}_x}$ is the structured representation of the other observed attributes $\mathbf{X}$ with respect to the conceptual level causal graph $\mathcal{G}_c$. We have $\widehat{Y}(\mathbf{Z^{'}_x})$ satisfy counterfactual fairness. 
\end{theorem}
\begin{proof}
Given the causal graph $\mathcal{G}$ as shown in Figure~\ref{pic:proof}, there is not a parent node of $\mathbf{A}$ in $\mathbf{X}$, and there is not a child node of $Y$ in $\mathbf{X}$. $\mathbf{X}$ contains four subsets: $\mathbf{X}^{\mathbf{A}}_{Y}$ is the subset of other observed attributes that are descendants of $\mathbf{A}$ and parents of $Y$; $\mathbf{X}^{\mathbf{N}}_{Y}$ is the subset of other observed attributes that are only parents of $Y$; $\mathbf{X}^{\mathbf{N}}_{\mathbf{N}}$ is the subset of other observed attributes that are no relationships with $\mathbf{A}$ and $Y$; $\mathbf{X}^{\mathbf{A}}_{\mathbf{N}}$ is the subset of other observed attributes that are only descendants of $\mathbf{A}$. After perfect representation learning, we obtain $\mathbf{Z_a}$ and $\mathbf{Z^{'}_x}$.

We proof that $\mathbf{Z^{'}_x}$ is not the descendant of $\mathbf{A}$ with the following two subsets. For the first subsets $\{\mathbf{X}^{\mathbf{A}}_{Y},\mathbf{X}^{\mathbf{N}}_{Y},\mathbf{X}^{\mathbf{A}}_{\mathbf{N}}\}$, there are seven paths between $\mathbf{A}$ and $\mathbf{Z^{'}_x}$, including $\mathbf{A} \to \mathbf{X}^{\mathbf{A}}_{Y} \gets \mathbf{Z^{'}_x}$, $\mathbf{A} \to \mathbf{X}^{\mathbf{A}}_{Y} \to Y \gets \mathbf{Z^{'}_x}$, $\mathbf{A} \to \mathbf{X}^{\mathbf{A}}_{Y} \to Y \gets \mathbf{X}^{\mathbf{N}}_{Y} \gets \mathbf{Z^{'}_x}$,  $\mathbf{A} \to Y \gets \mathbf{X}^{\mathbf{A}}_{Y} \gets \mathbf{Z^{'}_x}$, $\mathbf{A} \to Y \gets \mathbf{Z^{'}_x}$, $\mathbf{A} \to Y \gets \mathbf{X}^{\mathbf{N}}_{Y} \gets \mathbf{Z^{'}_x}$ and  $\mathbf{A} \to \mathbf{X}^{\mathbf{A}}_{\mathbf{N}} \gets Y$. These seven paths are blocked by $\emptyset$ (i.e., $\mathbf{A}$ and $\mathbf{Z^{'}_x}$ are $d$-separated by $\emptyset$), since each path contains a collider either $\mathbf{X}^{\mathbf{A}}_{Y}$ or $Y$ or $\mathbf{X}^{\mathbf{A}}_{\mathbf{N}}$. For second subset $\mathbf{X}^{\mathbf{N}}_{\mathbf{N}}$, there is no path connecting $\mathbf{X}^{\mathbf{N}}_{\mathbf{N}}$ and $Y$. Hence, $\mathbf{Z^{'}_x}$ is not the descendant of $\mathbf{A}$. Therefore, $\widehat{Y}(\mathbf{Z^{'}_x})$ is counterfactually fair based on Proposition~\ref{def:implication}. 
\end{proof}
\begin{figure}[t]
	\centering
	\includegraphics[scale=0.43]{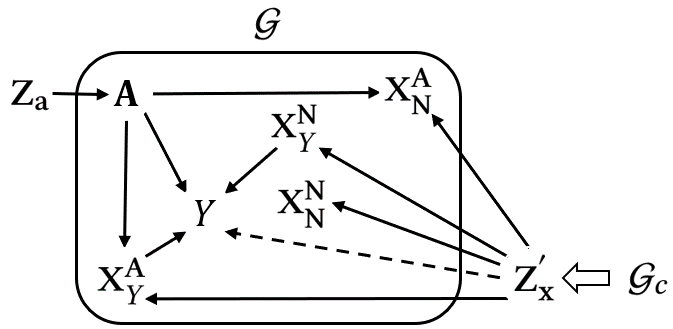}
	\caption{$\mathcal{G}$ is the causal graph that represents the causal relationship between $\mathbf{A}$, $\mathbf{X} = \{\mathbf{X}^{\mathbf{A}}_{Y},\mathbf{X}^{\mathbf{N}}_{Y},\mathbf{X}^{\mathbf{A}}_{\mathbf{N}},\mathbf{X}^{\mathbf{N}}_{\mathbf{N}}\}$ and $Y$. The dotted line represents the prediction process that uses $\mathbf{Z^{'}_x}$.}
	\label{pic:proof}
\end{figure}

We use Figure~\ref{pic:proof} to show whether the following predictors satisfy counterfactual fairness.
	\begin{itemize}
		\item $\widehat{Y}(\mathbf{A},\mathbf{X})$: This model is unfair since it uses sensitive attributes to make prediction.
		\item $\widehat{Y}(\mathbf{X})$: This model satisfies fairness through awareness~\cite{dwork2012fairness} but fails to achieve counterfactual fairness. The predictor $\widehat{Y}(\mathbf{X})$ does not use sensitive attributes explicitly, but it uses $\mathbf{X}^{\mathbf{A}}_{Y}$ and $\mathbf{X}^{\mathbf{A}}_{\mathbf{N}}$ which are the descendants of $\mathbf{A}$.
		\item $\widehat{Y}(\mathbf{Z_a},\mathbf{Z^{'}_x})$: This model is unfair because it uses sensitive attributes for prediction. The reason is that $\mathbf{Z_a}$ is the representation of $\mathbf{A}$, which should be consider as sensitive attributes either.
		\item $\widehat{Y}(\mathbf{X}^{\mathbf{N}}_{Y},\mathbf{X}^{\mathbf{N}}_{\mathbf{N}})$: This model satisfies counterfactual fairness since both $\mathbf{X}^{\mathbf{N}}_{Y}$ and $\mathbf{X}^{\mathbf{N}}_{\mathbf{N}}$ are non-descendants of $\mathbf{A}$. However, this predictor losses a lot of useful information that embeds in other observed attributes, which means it may not achieve an acceptable prediction accuracy.
		\item $\widehat{Y}(\mathbf{Z^{'}_x})$: This model is counterfactually fair based on Theorem~\ref{theo:002} and achieves higher accuracy than $\widehat{Y}(\mathbf{X}^{\mathbf{N}}_{Y},\mathbf{X}^{\mathbf{N}}_{\mathbf{N}})$ as shown in our experiments.
	\end{itemize}

\subsection{CF-VAE}
We first discuss the causal constraints and then explain the loss function of CF-VAE in detail. The architecture of CF-VAE is shown in Figure~\ref{pic:CFVAE1}.
\begin{figure*}[t]
	\centering
	\includegraphics[scale=0.43]{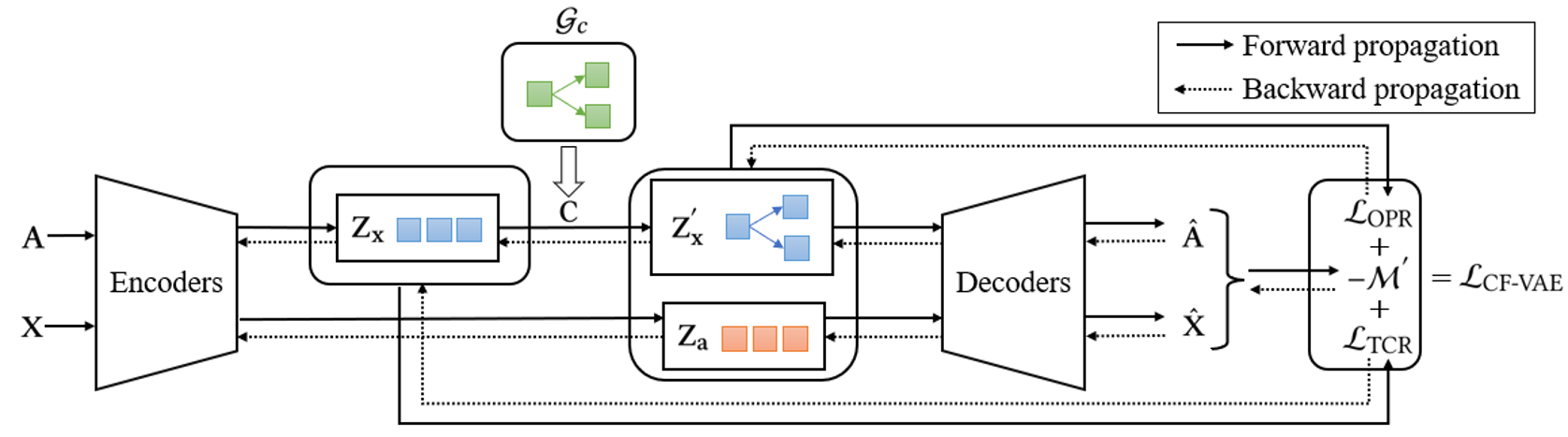}
	\caption{The architecture of CF-VAE. The adjacency Matrix $\mathbf{C}$ is used to construct $\mathbf{Z^{'}_x}$ and is determined by the conceptual level causal graph $\mathcal{G}_c$ with respect to domain knowledge. $\mathcal{L}_{\text{TCR}}$ is used to ensure that each attribute in $\mathbf{Z_x}$ is independent of each other. $\mathcal{L}_{\text{OPR}}$ is used to encourage that $\mathbf{Z_x^{'}}$ do not contain sensitive information. The loss function of CF-VAE is $\mathcal{L}_{\text{CF-VAE}}$.}
	\label{pic:CFVAE1}
\end{figure*} 

\subsubsection{Learning Representations with Causal Constraints}\label{3.2.1}
We aim to retain causal relationships between ``concepts'' through a more easily accessible conceptual level causal graph $\mathcal{G}_c$ and embed these relationships in representations. Following the works in~\cite{ZhengARX18,YangLCSHW21}, we transform these causal relationships in the form of an adjacency matrix $\mathbf{C}$ as causal constraints to construct $\mathbf{Z^{'}_x}$ and feed them to predictive models.

Many researchers study fair decision-making under the method of causality, in which an accessible causal graph is an important assumption. \citet{zhang2017achieving} used the PC algorithm~\cite{spirtes2000causation, kalisch2007estimating} to learn the causal relationships of the dataset itself and used the causal graph to restrict the transfer of sensitive information along the specific paths. Learning causal graphs from observational data has been shown to be feasible in unbiased data, but in fairness computing, the dataset may be biased and the PC algorithm may not be suitable. Other researches~\cite{nabi2018fair,chiappa2019path} assume that the complete version of causal graph is accessible and evolved from domain knowledge. In our paper, we adopt the second approach and further weaken it. We assume that $\mathcal{G}_c$ only covers the relationships between ``concepts'', not all observed attributes, which is easy to obtain consensus from experts.

To formalise causal relationships, we consider $n$ ``concepts'' in the dataset, which means $\mathbf{Z^{'}_x}$ should have the same dimension as ``concepts''. The ``concepts'' in observations are causally structured by $\mathcal{G}_c$ with an adjacency matrix $\mathbf{C}$. For simplicity, in this paper, the causal constraints are exactly implemented by a linear structural equation model as follows:
\begin{equation}
	\begin{aligned}
		\mathbf{Z_x^{'}} = (\mathbf{I}-\mathbf{C}^T)^{-1}\mathbf{Z_x},
	\end{aligned}	
\end{equation}where $\mathbf{I}$ is the identity matrix, $\mathbf{Z_x}$ is obtained from the encoder, $\mathbf{Z_x^{'}}$ is constructed from $\mathbf{Z_x}$ and $\mathbf{C}$. $\mathbf{C}$ is obtained from $\mathcal{G}_c$ with respect to domain knowledge. The parameters in $\mathbf{C}$ indicate that there are corresponding edges, and the values of the parameters indicate the weight of the causal relationships. It is worth noting that if the parameter value is zero, it means that such an edge does not exist, i.e., no causal relationship between these two ``concepts''.

As mentioned above, $\mathbf{Z_x}$ is obtained from the encoder, we cannot guarantee that each attribute inside is independent. To ensure the independence of each attribute in $\mathbf{Z_x}$, we follow the Factor-VAE in work~\cite{kim2018disentangling} and employ the total correction regularisation (TCR) in our loss function. TCR also encourages the correctness of structured $\mathbf{Z_x^{'}}$ with respect to domain knowledge since there are no correlations in $\mathbf{Z_x}$ before adding causal constraints. The TCR for our proposed CF-VAE is defined as:
\begin{equation}
	\begin{aligned}
		\mathcal{L}_{\text{TCR}} = \gamma D_{KL}[q(\mathbf{Z_x})||\prod_{i=1}^{D_{\mathbf{Z_x}}} q(Z_{\mathbf{x}_i})],
	\end{aligned}	
\end{equation}where $\gamma$ is the weight value, $D_{\mathbf{Z_x}}$ is dimension of $\mathbf{Z_x}$.

\subsubsection{Learning Strategy} \label{3.2.2}We first explain the architecture of CF-VAE without using causal constraints. Then, we add causal constraints and orthogonality promoting regularisation (OPR) to obtain the loss function of CF-VAE.

In the inference model, the variational approximations of the posteriors are defined as:
\begin{equation}
	\begin{aligned}
		&q(\mathbf{Z_a}|\mathbf{A}) = \prod_{i=1}^{D_{\mathbf{Z_a}}}\mathcal{N}(\mu = \hat{\mu}_{Z_{\mathbf{a}_i}}, \sigma^2 = \hat{\sigma}^2_{Z_{\mathbf{a}_i}});\\
		&q(\mathbf{Z_x}|\mathbf{X}) = \prod_{i=1}^{D_{\mathbf{Z_x}}}\mathcal{N}(\mu = \hat{\mu}_{Z_{\mathbf{x}_i}}, \sigma^2 = \hat{\sigma}^2_{Z_{\mathbf{x}_i}}),
	\end{aligned}	
\end{equation} where $\hat{\mu}_{a_i},\hat{\mu}_{x_i}$ and $\hat{\sigma}^2_{a_i},\hat{\sigma}^2_{x_i}$ are the means and variances of the Gaussian distributions parameterised by neural networks.

The generative model for $\mathbf{A}$ and $\mathbf{X}$ are defined as:
\begin{equation}
	\begin{aligned}
		p(\mathbf{A}|\mathbf{Z_a}) = \prod_{i=1}^{D_{\mathbf{A}}} p(A_i|\mathbf{Z_a});~p(\mathbf{X}|\mathbf{Z_x}) = \prod_{i=1}^{D_{\mathbf{X}}} p(X_i|\mathbf{Z_x}),
	\end{aligned}	
\end{equation} where $D_{\mathbf{A}}$ and $D_{\mathbf{X}}$ are dimensions of $\mathbf{A}$ and $\mathbf{X}$.

Following the setting in VAE~\cite{KingmaW13}, we choose Gaussian distribution as prior distributions, which are defined as:
\begin{equation}
	\begin{aligned}
		p(\mathbf{Z_a}) = \prod_{i=1}^{D_{\mathbf{Z_a}}}\mathcal{N}(Z_{\mathbf{a}_i} | 0,1);~p(\mathbf{Z_x}) = \prod_{i=1}^{D_{\mathbf{Z_x}}} \mathcal{N}(Z_{\mathbf{x}_i} | 0,1).
	\end{aligned}	
\end{equation}

Given the training samples, the parameters can be optimised by maximising the following ELBO:

\begin{equation}
	\begin{aligned}
		\mathcal{M} =~ &\mathbb{E}_{q(\mathbf{Z_a}|\mathbf{A})}[\log p(\mathbf{A}|\mathbf{Z_a})] + \mathbb{E}_{q(\mathbf{Z_x}|\mathbf{X})}[\log p(\mathbf{X}|\mathbf{Z_x})] \\&- D_{KL}[q(\mathbf{Z_a}|\mathbf{A})||p(\mathbf{Z_a})] - D_{KL}[q(\mathbf{Z_x}|\mathbf{X})||p(\mathbf{Z_x})].
		\label{eqa:ELBO1}
	\end{aligned}	
\end{equation}

We note that Equation~\ref{eqa:ELBO1} is not under causal constraints and still using $\mathbf{Z_x}$ to optimise. The $\mathcal{M}$ comprises four terms: the first and second term denote the reconstruction loss between the original $\{\mathbf{A}, \mathbf{X}\}$ and $\{\hat{\mathbf{A}},\hat{\mathbf{X}}\}$; the third term and the fourth term are used for calculating the distribution distance between the prior knowledge and the latent representations that we obtained. 

We follow Section~\ref{3.2.1} and add causal constraints in Equation~\ref{eqa:ELBO1}. The updated ELBO is defined as:
\begin{equation}
	\begin{aligned}
		\mathcal{M}^{'} =~ &\mathbb{E}_{q(\mathbf{Z_a}|\mathbf{A})}[\log p(\mathbf{A}|\mathbf{Z_a})] + \mathbb{E}_{q(\mathbf{Z_x^{'}}|\mathbf{X})}[\log p(\mathbf{X}|\mathbf{Z_x^{'}})] \\&- D_{KL}[q(\mathbf{Z_a}|\mathbf{A})||p(\mathbf{Z_a})] - D_{KL}[q(\mathbf{Z_x^{'}}|\mathbf{X})||p(\mathbf{Z_x^{'}})],
	\end{aligned}	
	\label{eqa:ELBO2}
\end{equation}where 
\begin{equation*}
	\begin{aligned}
		&p(\mathbf{Z_x^{'}}) = (\mathbf{I}-\mathbf{C}^T)^{-1}p(\mathbf{Z_x}); ~p(\mathbf{X}|\mathbf{Z^{'}_x}) = \prod_{i=1}^{D_{\mathbf{X}}} p(X_i |\mathbf{Z^{'}_x});\\
		&q(\mathbf{Z^{'}_x}|\mathbf{X}) = \prod_{i=1}^{D_{\mathbf{Z^{'}_x}}} \mathcal{N}(\mu = \hat{\mu}_{Z^{'}_{\mathbf{x}_i}}, \sigma^2 = \hat{\sigma}^2_{Z^{'}_{\mathbf{x}_i}}).
	\end{aligned}
\end{equation*}

We introduce orthogonality to encourage disentanglement between $\mathbf{Z_a}$ and $\mathbf{Z_x^{'}}$. Following the work in \cite{yu2011diversity}, we employ orthogonality promoting regularisation based on the pairwise cosine similarity among latent representations: if the cosine similarity is close to zero, then the latent representations are closer to being orthogonal and independent. The cosine similarity (CS) is defined as:
\begin{equation}
	\begin{aligned}
		CS(\mathbf{E_1}, \mathbf{E_2}) =  \frac{\mathbf{E_1}^{T} \mathbf{E_2}}{\|\mathbf{E_1}\|_2~ \|\mathbf{E_2}\|_2},
	\end{aligned}	
\end{equation}where $\|\mathbf{\cdot}\|_2$ is the $\mathit{l}_2$ norm.

To encourage orthogonality between two vectors $\mathbf{E_1}$ and $\mathbf{E_2}$, we can make their inner product $\mathbf{E_1}^{T} \mathbf{E_2}$ close to zero and their $\mathit{l}_2$ norm $\|\mathbf{E_1}\|_2$, $\|\mathbf{E_2}\|_2$ close to one~\cite{xie2018orthogonality}. The orthogonality promoting regularisation (OPR) for our proposed CF-VAE is defined as:
\begin{equation}
	\begin{aligned}
		\mathcal{L}_{\text{OPR}} = \frac{1}{B}\sum_{i=1}^{B}CS(\mathbf{Z_{a_{\mathit{i}}}}, \mathbf{Z_{x_{\mathit{i}}}^{'}}),
	\end{aligned}	
\end{equation}where $B$ denotes the batch size for neural network. 

In conclusion, the loss function of our proposed CF-VAE is defined as:
\begin{equation}
	\begin{aligned}
		\mathcal{L}_{\text{CF-VAE}} = -\mathcal{M}^{'} + \mathcal{L}_{\text{TCR}} + \mathcal{L}_{\text{OPR}}.
	\end{aligned}	
	\label{eq:CF-VAE}
\end{equation}

\section{Experiments}
In this section, we conduct extensive experiments to evaluate CF-VAE on real-world datasets. Before showing the detailed results, we first present the details of selected methods and the evaluation metrics.

\begin{figure*}[t]
	\begin{subfigure}[b]{0.4090\textwidth}
		\raggedright
		\includegraphics[scale=0.46]{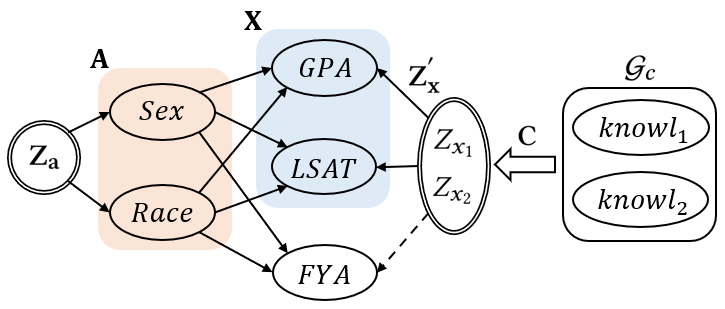}
		\caption{}
		\label{pic:law}
	\end{subfigure}
	\begin{subfigure}[b]{0.5401\textwidth}
		\raggedleft
		\includegraphics[scale=0.46]{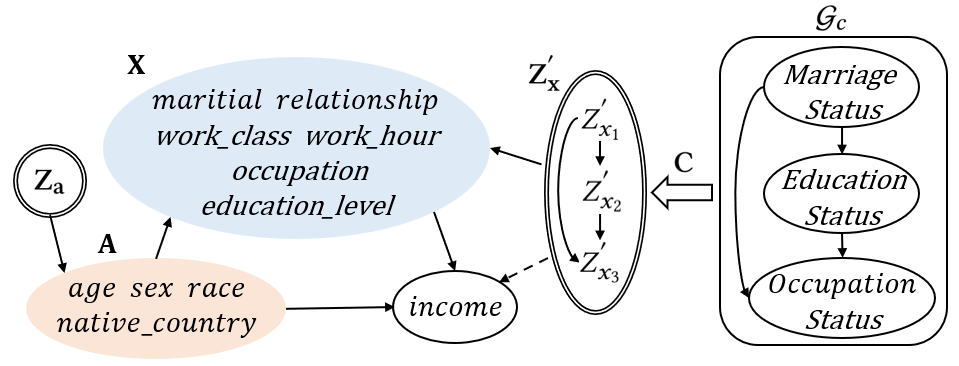}
		\caption{}
		\label{pic:adult1}
	\end{subfigure}
	\caption{(a) The process of CF-VAE for Law school dataset. (b) The process of CF-VAE for Adult dataset. $\mathbf{Z_a}$ is the representation of $\mathbf{A}$; $\mathbf{Z_x^{'}}$ is the structured representation of $\mathbf{X}$. The adjacency matrix $\mathbf{C}$ is used to construct $\mathbf{Z_x^{'}}$ with respect to $\mathcal{G}_c$. The dotted line represents the prediction process of $Y$ that uses $\mathbf{Z_x^{'}}$.}
	\label{pic:experiment}
\end{figure*}

\subsection{Framework Comparison}
The proposed CF-VAE is considered as a pre-processing technique to address fairness issues since it obtains structured representations for downstream predictive models to achieve counterfactual fairness. Hence, we compare CF-VAE with traditional and VAE-based pre-processing methods. For traditional methods, we select baselines including ReWeighting (RW)~\cite{kamiran2012data}, Disparate Impart Remover (DIR)~\cite{feldman2015certifying} and Optimized Preprocessing (OP)~\cite{CalmonWVRV17}. Both of them are available in AIF360~\cite{bellamy2019ai}. For VAE-based methods, we compare with VFAE~\cite{LouizosSLWZ15} and FFVAE~\cite{creager2019flexibly}. Both of them are implemented in Pytorch~\cite{NEURIPS2019_9015}. We also obtain the Full model for comparison, which uses all attributes in the dataset to make predictions.

We do not choose the basic VAE (e.g., VAE~\cite{KingmaW13}, $\beta$-VAE~\cite{HigginsMPBGBML17}, and Factor-VAE~\cite{kim2018disentangling}) for comparison in this experiment, since they are not optimised for fairness problems. In addition, we do not use VAE-based inference models~\cite{chiappa2019path, SarhanNEA20, KimSJSJKM21,YangLCSHW21} for comparison, because the purpose of these inference models is to generate counterfactual data or to estimate effects, which is different from our goals.

We select several well-known predictive models to simulate the downstream prediction process. Linear Regression ($\mathrm{LR_{R}}$), Stochastic Gradient Descent Regression ($\mathrm{SGD_{R}}$) and Multi-layer Perceptron Regression ($\mathrm{MLP_{R}}$) are used for regression tasks; Logistic Regression ($\mathrm{LR_{C}}$), Stochastic Gradient Descent Classification ($\mathrm{SGD_{C}}$) and Multi-layer Perceptron Classification ($\mathrm{MLP_{C}}$) are used for classification tasks. For each predictive model, we run 10 times and record the mean and error of the results for evaluation metrics, which are explained in detail in Section~\ref{Metrics}.

\subsection{Evaluation Metrics}\label{Metrics}
\subsubsection{Fairness}
There are no metrics to quantify counterfactual fairness since we can only obtain real-world data. We propose the situation test to measure fairness for different predictive models. The situation test has already been widely used in United States to detect individual discrimination~\cite{bendick2007situation}. In our experiment, we construct a matched pair for each individual by inverting the values of sensitive attributes. We take this matched pair as the input to the predictive model, and the predictive model is fair if the predictions of the matched pair are the same as the original pair.

We define unfairness score (UFS) to measure the result of the situation test. Specifically, the form of score differs for different predictive models. For regression tasks, we define $\mathrm{UFS_{R}}$ that measure the bias between prediction results for the matched pair and the original pair; For classification tasks, $\mathrm{UFS_{C}}$ is defined as how many individuals' prediction results are changed after intervening the values of sensitive attributes. $\mathrm{UFS_{R}}$ and $\mathrm{UFS_{C}}$ are described as follows:
\begin{equation}
	\begin{aligned}
		&\mathrm{UFS_{R}} = \sqrt{\frac{1}{N} \sum_{i=1}^{N}\Big ( \widehat{Y}_{\mathbf{A}\leftarrow \mathbf{a}}(\mathbf{Z^{'}_{x_{\mathit{i}}}}) - \widehat{Y}_{\mathbf{A}\leftarrow \bar{\mathbf{a}}}(\mathbf{Z^{'}_{x_{\mathit{i}}}}) \Big)^{2}}~;\\&\mathrm{UFS_{C}} = \frac{1}{N} \sum_{i=1}^{N} \mathrm{xor}\Big(\widehat{Y}_{\mathbf{A}\leftarrow \mathbf{a}}(\mathbf{Z^{'}_{x_{\mathit{i}}}}), \widehat{Y}_{\mathbf{A}\leftarrow \bar{\mathbf{a}}}(\mathbf{Z^{'}_{x_{\mathit{i}}}})\Big),
	\end{aligned}	
\end{equation}where $N$ is the number of samples for evaluation.

The lower $\mathrm{UFS}$ value means that the predictive models achieve higher individual fairness.

\subsubsection{Accuracy} 
We evaluate the performance on prediction with the following metrics. For regression tasks, we use Root Mean Square Error (RMSE) to compare the error between prediction results and target attributes' values. For classification tasks, we use accuracy to evaluate various predictive models.

\begin{table*}[t]
	\caption{The results for Law School dataset. The best fairness aware $\mathrm{RMSE}$ and the best $\mathrm{UFS_{R}}$ are shown in bold.}
	\begin{tabular}{ccccccccccc}
		\toprule
		\multirow{2}{*}{\quad Model \quad} &  & \multicolumn{3}{c}{Accuracy ($\mathrm{RMSE}$) $\downarrow$} &  & \multicolumn{3}{c}{Fairness ($\mathrm{UFS_{R}}$) $\downarrow$} \\ \cline{3-5} \cline{7-9}  
		& & $\mathrm{LR_{R}}$ & $\mathrm{SGD_{R}}$ & $\mathrm{MLP_{R}}$ & &  $\mathrm{LR_{R}}$ & $\mathrm{SGD_{R}}$ & $\mathrm{MLP_{R}}$ \\ \cline{1-1} \cline{3-5} \cline{7-9} 
		Full &  & 0.865 ± 0.007 & 0.867 ± 0.007 & 0.865 ± 0.007 &  & 0.660 ± 0.019 & 0.762 ± 0.019 & 0.760 ± 0.045 \\ \cline{1-1} \cline{3-5} \cline{7-9}
		RW &  & 0.955 ± 0.013 & 0.956 ± 0.012 & 0.953 ± 0.012 &  & 0.067 ± 0.002 & 0.067 ± 0.001 & 0.079 ± 0.003 \\
		DIR &  & 0.943 ± 0.009 & 0.944 ± 0.009 & 0.941 ± 0.010 &  & 0.060 ± 0.001 & 0.060 ± 0.001 & 0.070 ± 0.002 \\
		OP &  & 0.959 ± 0.011 & 0.960 ± 0.011 & 0.956 ± 0.010 &  & 0.047 ± 0.001 & 0.046 ± 0.001 & 0.055 ± 0.003 \\
		VFAE &  & 0.932 ± 0.007 & 0.933 ± 0.007 & 0.934 ± 0.007 &  & 0.035 ± 0.010 & 0.074 ± 0.017 & 0.096 ± 0.010 \\
		FFVAE &  & 0.933 ± 0.005 & 0.934 ± 0.004 & 0.935 ± 0.005 &  & 0.032 ± 0.007 & 0.060 ± 0.022 & 0.097 ± 0.008 \\ 
		CF-VAE &  & \textbf{0.931 ± 0.006} & \textbf{0.932 ± 0.006} & \textbf{0.932 ± 0.006} &  & \textbf{0.013 ± 0.006} & \textbf{0.025 ± 0.011} & \textbf{0.044 ± 0.006} \\\bottomrule
	\end{tabular}
	\label{tab:law}
\end{table*}

\begin{table*}[t]
	\caption{The results for Adult dataset. The best fairness aware accuracy and the best $\mathrm{UFS_{C}}$ are shown in bold.}
	\begin{tabular}{ccccccccccc}
		\toprule
		\multirow{2}{*}{\quad Model \quad} &  & \multicolumn{3}{c}{Accuracy $\uparrow$} &  & \multicolumn{3}{c}{Fairness ($\mathrm{UFS_{C}}$) $\downarrow$} \\ \cline{3-5} \cline{7-9}  
		& & $\mathrm{LR_{C}}$ & $\mathrm{SGD_{C}}$ & $\mathrm{MLP_{C}}$ & &  $\mathrm{LR_{C}}$ & $\mathrm{SGD_{C}}$ & $\mathrm{MLP_{C}}$ \\ \cline{1-1} \cline{3-5} \cline{7-9} 
		Full &  & 0.802 ± 0.002 & 0.803 ± 0.004 & 0.831 ± 0.004 &  & 0.068 ± 0.003 & 0.060 ± 0.018 & 0.034 ± 0.009 \\ \cline{1-1} \cline{3-5} \cline{7-9}
		RW &  & 0.797 ± 0.001 & 0.792 ± 0.002 & 0.819 ± 0.001 &  & 0.038 ± 0.001 & 0.029 ± 0.002 & 0.052 ± 0.001 \\
		DIR &  & 0.800 ± 0.001 & 0.793 ± 0.003 & 0.817 ± 0.001 &  & 0.035 ± 0.001 & 0.027 ± 0.002 & 0.046 ± 0.001 \\
		OP &  & 0.780 ± 0.002 & 0.779 ± 0.003 & 0.783 ± 0.002 &  & 0.032 ± 0.003 & 0.030 ± 0.004 & 0.033 ± 0.005 \\
		VFAE &  & 0.785 ± 0.001 & 0.781 ± 0.003 & 0.819 ± 0.004 &  & 0.062 ± 0.002 & 0.041 ± 0.010 & 0.025 ± 0.003 \\
		FFVAE &  & 0.785 ± 0.003 & 0.782 ± 0.001 & 0.814 ± 0.005 &  & 0.062 ± 0.001 & 0.044 ± 0.010 & 0.032 ± 0.010 \\ 
		CF-VAE &  & \textbf{0.801 ± 0.002} & \textbf{0.794 ± 0.004} & \textbf{0.820 ± 0.002} &  & \textbf{0.031 ± 0.002} & \textbf{0.020 ± 0.006} & \textbf{0.024 ± 0.004} \\\bottomrule
	\end{tabular}
	\label{tab:adult}
\end{table*}

\subsection{Law School}
The law school dataset comes from a survey~\cite{wightman1998lsac} of admissions information from 163 law schools in the United States. It contains information of 21,790 law students, including their entrance exam scores (LSAT), their grade point average (GPA) collected prior to law school, and their first-year average grade (FYA). The school expects to predict if the applicants will have a high FYA. Gender and race are sensitive attributes in this dataset, and the school also wants to ensure that predictions are not affected by sensitive attributes. However, LSAT, GPA and FYA scores may be biased due to socio-environmental factors. The process of CF-VAE for the Law school dataset is shown in Figure~\ref{pic:law}.

\subsubsection{Implementation Details} We divide the Law school dataset into 70\% training set for training the representation models, 30\% testing set for evaluating the accuracy of the predictive models, and inverting the values of sensitive attributes in the testing set to generate the auditing set for evaluating the fairness of the predictive models. 

We use the same $\mathcal{G}_c$ as shown in work~\cite{kusner2017counterfactual} to model latent ``concepts'' of $GPA$ and $LSAT$. Since $knowl_{1}$ and $knowl_{2}$ have no causal relationship, the parameters in adjacency matrix $\mathbf{C}$ are set to zero. As a results, we set the {\small $D_{\mathbf{Z^{'}_x}}$}= 2 and set the weight value in $\mathcal{L}_{\text{TCR}}$ as $\gamma = 10$.

\subsubsection{Fairness} The purpose is to demonstrate our method can achieve better fairness performance than other VAE-based methods. As shown in Table~\ref{tab:law}, since the Full model uses sensitive attributes to make predictions, inverting sensitive attributes has the highest impact on the individual's prediction results, which means that the model is unfair. RW, DIR and OP achieves fair predictions by modifying the dataset compared to the Full model. Both VFAE and FFVAE disentangle the sensitive attributes with latent representations, so the influence of inverting the sensitive attributes on the prediction results is small. Our method achieves the lowest $\mathrm{UFS_{R}}$, $0.013$, $0.025$, and $0.044$ for $\mathrm{LR_{R}}$, $\mathrm{SGD_{R}}$, and $\mathrm{MLP_{R}}$ respectively, which means CF-VAE disentangle $\mathbf{Z}^{'}_x$ and  $\mathbf{Z_a}$ more precisely. 

\subsubsection{Accuracy} The accuracy results are shown in Table~\ref{tab:law}. The Full model is unfair and it uses sensitive information to more accurately predict FYA and thus achieves the highest accuracy. The proposed CF-VAE achieves the best fairness aware accuracy in all predictive models than other methods. Our method not only achieves counterfactual fairness for downstream predictors but also flexible for choosing predictive models. 

\subsection{Adult}
The Adult dataset comes from the UCI repository~\cite{asuncion2007uci} contains 14 attributes including race, age, education information, marital information as well as capital gain and loss for 48,842 individuals. The process of CF-VAE is shown in Figure~\ref{pic:adult1}.

\subsubsection{Implementation Details}
We pre-process the dataset by deleting missing information and encoding discrete attributes. After that, we get 45,222 individuals and the downstream tasks' goal is to predict whether the individual's income is above \$50,000. We set $race$, $age$, $sex$ and $native~country$ as $\mathbf{A}$; $maritial$, $relationship$, $work~class$, $work~hour$, $occupation$ and $education~level$ as $\mathbf{X}$. We divide the Adult dataset into 70\% training set for training representation models, 30\% testing set for evaluating the accuracy of the predictive models, and select $10,000$ individuals with female that income below 50K as the auditing set.

We use the same $\mathcal{G}_c$ as shown in previous research~\cite{nabi2018fair,chiappa2019path} to model the latent ``concepts''. We set the {\small $D_{\mathbf{Z^{'}_x}}$}= 3 and set the weight value in $\mathcal{L}_{\text{TCR}}$ as $\gamma = 10$. The adjacency matrix $\mathbf{C}$ is defined as:
\begin{equation}
	\begin{aligned}
		\mathbf{C} = 
		\left|                
		\begin{array}{ccc}   
			0&\lambda_{12}&\lambda_{13}\\  
			0&0&\lambda_{23}\\ 
			0&0&0\\  
		\end{array}
		\right|
	\end{aligned}	
\end{equation} 

Then, we construct $\mathbf{Z_x^{'}}$ from $\mathbf{Z_x}$ and $\mathbf{C}$ as follows:
\begin{equation}
	\begin{aligned}
		&Z_{x_1}^{'} = Z_{x_1};~Z_{x_2}^{'} = \lambda_{12}Z_{x_1} + Z_{x_2};\\
		&Z_{x_3}^{'} = \lambda_{13}Z_{x_1} + \lambda_{23}Z_{x_2} + Z_{x_3}.
	\end{aligned}	
\end{equation}

We set parameter $\{\lambda_{12} = 1,\lambda_{13} = 1,\lambda_{23} = 1\}$ to denote that edges within latent representations, i.e., {\small $Z_{x_1}^{'} \to Z_{x_2}^{'}, Z_{x_1}^{'} \to Z_{x_3}^{'}, Z_{x_2}^{'} \to Z_{x_3}^{'}$}.

\subsubsection{Fairness}
The fairness results are shown in Table~\ref{tab:adult}, the Full model achieves the worst $\mathrm{UFS_{C}}$, since it use $\mathbf{A}$ to predict $income$. Both baseline fairness models and other VAE-based methods improve fairness to a certain extent. The proposed CF-VAE achieves the best $\mathrm{UFS_{C}}$, only $3.1\%$, $2.0\%$ and $2.4\%$ of individuals' results are affected by sensitive attributes' values inversions in $\mathrm{LR_{C}}$, $\mathrm{SGD_{C}}$ and $\mathrm{MLP_{C}}$, respectively. Our method achieves better fairness performance than other methods, since it remains causal relationships in latent representations with respect to $\mathcal{G}_c$ and disentangles structured representations with sensitive attributes.

\begin{table*}[t]
	\caption{The results of ablation study. The Full model and VFAE are shown in the first two rows. The third row is the method without causal constraints. The fourth row is the method without employing $\text{OPR}$. Our proposed CF-VAE is shown in the last row. The best fairness aware $\mathrm{RMSE}$ and the best $\mathrm{UFS_{R}}$ are shown in bold, and the runner-up results are underlined.}
	\begin{tabular}{ccccccccccc}
		\toprule
		\multirow{2}{*}{Loss function} &  & \multicolumn{3}{c}{Accuracy ($\mathrm{RMSE}$) $\downarrow$} &  & \multicolumn{3}{c}{Fairness ($\mathrm{UFS_{R}}$) $\downarrow$} \\ \cline{3-5} \cline{7-9}  
		& & $\mathrm{LR_{R}}$ & $\mathrm{SGD_{R}}$ & $\mathrm{MLP_{R}}$ & &  $\mathrm{LR_{R}}$ & $\mathrm{SGD_{R}}$ & $\mathrm{MLP_{R}}$ \\ \cline{1-1} \cline{3-5} \cline{7-9} 
		- &  & 0.078 ± 0.001 & 0.081 ± 0.001 & 0.081 ± 0.001 &  & 0.102 ± 0.001 & 0.098 ± 0.001 & 0.106 ± 0.002 \\ \cline{1-1} \cline{3-5} \cline{7-9}
		$-\mathcal{M}_{\text{VFAE}}$ &  & 0.126 ± 0.002 & 0.126 ± 0.002 & 0.145 ± 0.002 &  & 0.006 ± 0.001 & 0.010 ± 0.002 & 0.104 ± 0.005 \\
		$-\mathcal{M}$ &  & 0.125 ± 0.001 & 0.125 ± 0.001 & 0.145 ± 0.001 &  & 0.007 ± 0.001 & 0.011 ± 0.003 & 0.105 ± 0.003 \\
		$-\mathcal{M}^{'} + \mathcal{L}_{\text{TCR}}$ &  & \textbf{0.109 ± 0.001} & \underline{0.111 ± 0.001} & \underline{0.122 ± 0.002} &  & \underline{0.003 ± 0.001} & \textbf{0.004 ± 0.002} & \underline{0.071 ± 0.002} \\
		$-\mathcal{M}^{'} + \mathcal{L}_{\text{TCR}} + \mathcal{L}_{\text{OPR}}$ &  & \textbf{0.109 ± 0.001} & \textbf{0.110 ± 0.001} & \textbf{0.121 ± 0.001} &  & \textbf{0.002 ± 0.001} & \underline{0.005 ± 0.002} & \textbf{0.070 ± 0.002} \\ \bottomrule
	\end{tabular}
	\label{tab:Abla}
\end{table*}

\subsubsection{Accuracy}
The Full model uses all observed attributes for predictions. It is worth noting that the Full model does not achieve the 85\% accuracy shown in~\cite{asuncion2007uci}, because we omit capital gain and loss, and achieve similar accuracy as shown in the work~\cite{nabi2018fair}. 

The accuracy results are shown in Table~\ref{tab:adult}. In order to achieve fairness, VFAE and FFVAE lose about 2\% of their accuracy performance. RW, DIR and OP modify the dataset resulting in a loss of predictive performance. The proposed CF-VAE not only guarantees the fairness performance but also retains the causal relationships to improve accuracy. CF-VAE loses less information than other VAE-base methods and achieves the best fairness aware accuracy performance in all predictive models, i.e., $80.1\%$, $79.4\%$ and $82.0\%$ in $\mathrm{LR_{C}}$, $\mathrm{SGD_{C}}$ and $\mathrm{MLP_{C}}$, respectively.

\subsection{Ablation Study}
We follow the same procedure in~\cite{cheng2022toward} to generate synthetic datasets and conduct an ablation study to validate the contribution of each component in our method as shown in Table~\ref{tab:Abla}. 

The Full model uses all the observed attributes to train the predictors. The predictors achieve the best accuracy but the worst fairness performance as shown in the first row in Table~\ref{tab:Abla}. VFAE is the basic VAE-based unsupervised fair representation learning method. We set it to be the baseline in the second row in Table ~\ref{tab:Abla}. The third row is CF-VAE without adding causal constraints, which achieves similar results as VFAE since both methods remove sensitive information from the learnt representations. 

Then, we employ causal constraints and add TCR ($\gamma = 10$) in the loss function. As shown in the fourth row in Table ~\ref{tab:Abla}, this step retains causal relationships in latent representations and improves both accuracy and fairness performance than previous rows. The last step is to encourage $\mathbf{Z^{'}_x}$ and $\mathbf{Z_a}$ are disentangled by adding OPR. Our proposed CF-VAE achieves the best accuracy performance and $\mathrm{UFS_{R}}$ among most predictive models as shown in the last row in Table ~\ref{tab:Abla}.

\section{Related Works}
The machine learning literature has increasingly focused on exploring how algorithms can protect marginalised populations from unfair treatment. An important research area is how to quantify fairness, which can be divided into two categories, the statistical framework and the causal framework. 

In the statistical framework, Demographic parity was defined by \citet{zemel2013learning}, which is used to measure group level fairness. Other similar metrics include equalised odds~\cite{hardt2016equality}, predictive rate parity~\cite{zafar2017fairness}.  \citet{dwork2012fairness} proposed a measurement to quantify individual level fairness, that is, similar individuals should have similar treatments, and they use distance functions to measure how similar between individuals. In the causal framework, the (conditional) average causal effect is used to quantify fairness between groups~\cite{li2017discrimination}; Natural direct and natural indirect effects are used to quantify specific fairness~ \cite{zhang2017causal,nabi2018fair,zhang2018fairness}; When unfair causal paths are identified by domain knowledge, \citet{chiappa2019path} used the path-specific causal effects to quantify fairness on approved paths. For more related works, please refer to the literature review~\cite{zhang2017anti,corbett2018measure,mehrabi2021survey}. 

Our work is related to learning fair representations, which aims to encode data information into a lower space while removing sensitive information, and remaining causal relationships with respect to domain knowledge for building counterfactually fair predictive models. VAE~\cite{KingmaW13} and $\beta$-VAE~\cite{HigginsMPBGBML17}, as introduced in Section~\ref{VAE}, have inspired several studies in fair representation learning. \citet{LouizosSLWZ15} first introduced VAE for learning fair representation to disentangle the sensitive information and non-sensitive information, they proposed a semi-supervised method to encourage disentanglement by using ``Maximum Mean Discrepanc'' (MMD). However, \citet{zemel2013learning}, \citet{gitiaux2021learning} argued that in real-world applications, the organisations that collect the data cannot predict the downstream uses of the data and the models that might be used. Due to this, there are many following up works focusing on unsupervised learning fair representation. For example, \citet{creager2019flexibly} proposed an algorithm that can achieve group level fairness by adding demographic parity as a constraint in objection function; \citet{song2019learning} developed an information theory-based method for learning maximally expressive representations subject to fairness constraints that allows users to control the fairness of representations by specifying limits on unfairness.

Our approach combines counterfactual fairness and unsupervised representation learning to provide the proper representations to help predictive models achieve counterfactual fairness. We extend the definition of counterfactual fairness~\cite{kusner2017counterfactual} to the representation learning. Based on the current literature review, our work is the first method to use VAE-based techniques for unsupervised representation and satisfy counterfactual fairness. Furthermore, we innovatively embed domain knowledge into representations by adding causal constraints with respect to domain knowledge.

\section{Conclusion}
In this paper, we investigate unsupervised counterfactually fair representation learning and propose a novel method named CF-VAE which considers causal relationships with respect to domain knowledge. We theoretically demonstrate that the structured representations obtained by CF-VAE enable predictive models to achieve counterfactual fairness. Experimental results on real-world datasets show that CF-VAE achieves better accuracy and fairness performance on downstream predictive models than the benchmark fairness methods. Ablation study on synthetic datasets shows that causal constraints with total correction regularisation achieve better accuracy performance and orthogonality promoting regularisation encourages disentanglement with sensitive attributes.


\begin{acks}
Supported by Australian Research Council (DP200101210) and Postgraduate Research Scholarship of University of South Australia.
\end{acks}


\bibliographystyle{ACM-Reference-Format}
\bibliography{mybibliography}


\end{document}